\documentclass[10pt, twocolumn]{article}

\usepackage[numbers]{natbib}
\usepackage{amssymb}
\usepackage{latexsym}
\usepackage{amsmath}
\usepackage{amssymb}
\usepackage{amsthm}
\usepackage{xcolor}
\usepackage{comment}
\usepackage{subcaption}
\usepackage{tikz}
\usepackage{pgfplots}
\usepackage{algorithm}
\usepackage{algpseudocode}
\usepackage{cleveref}
\usepackage{url}
\usepackage{lipsum}
\usepackage{array}
\usepackage{comment}
\usepackage{subcaption}
\usepackage{etoolbox}
\usepackage{wasysym}
\usepackage{titling}
\usepackage[affil-it]{authblk}

\usepackage{url}
\usepackage{xcolor}
\definecolor{newcolor}{rgb}{.8,.349,.1}

\newtheorem{lemma}{Lemma}
\newtheorem{theorem}{Theorem}

\DeclareMathOperator{\sign}{sign}

\DeclareMathOperator*{\cost}{cost}
\DeclareMathOperator*{\logistic}{logistic}
\newcommand\bias{\mathit{bias}}

\renewrobustcmd{\bfseries}{\fontseries{b}\selectfont}
\newrobustcmd{\bc}{\bfseries}

\begin{document}

\thispagestyle{empty}
                                                             
%\begin{frontmatter}

\title{Limitation of Capsule Networks}

\author{David Peer \qquad
Sebastian Stabinger \qquad
Antonio Rodr\'{i}guez-S\'{a}nchez\\
University of Innsbruck \\
Austria\\
{\tt\small https://iis.uibk.ac.at/}
}

\maketitle

\begin{abstract}
A recently proposed method in deep learning groups multiple neurons to capsules such that each capsule represents an object or part of an object. Routing algorithms route the output of capsules from lower-level layers to upper-level layers. In this paper, we prove that state-of-the-art routing procedures decrease the expressivity of capsule networks. More precisely, it is shown that \emph{EM-routing} and \emph{routing-by-agreement} prevent capsule networks from distinguishing inputs and their negative counterpart. Therefore, only symmetric functions can be expressed by capsule networks, and it can be concluded that they are not universal approximators. We also theoretically motivate and empirically show that this limitation affects the training of deep capsule networks negatively. Therefore, we present an incremental improvement for state-of-the-art routing algorithms that solves the aforementioned limitation and stabilizes the training of capsule networks.
\end{abstract}

%\begin{keyword}
%\KWD Capsule Networks \sep Deep Learning \sep Limitation \sep Universal approximator

%\end{keyword}

%\end{frontmatter}

%\linenumbers

%%%%%%%%%%%%%%%%%%%%%%%%%%%%%%%%%%%%%%%%
\section{Introduction}\label{sec:introduction}
Capsules in capsule networks are groups of neurons that represent an object or a part of an object in a parse tree as introduced by \citet{transforming-autoencoders}. A capsule encodes \emph{instantiation parameters} of an object (e.g. its position, location, orientation, ...) through vector-based representations. The norm of this vector is called the \emph{activation} of a capsule and it encodes the probability of whether this object exists in the current input or not. Later, \citet{dynamic-routing} introduced the \emph{CapsNet} architecture, whose first layers construct primary capsules with ReLU convolutions. Those primary capsules are then routed to subsequent capsule layers with so-called routing algorithms. One such algorithm is the iterative \emph{routing-by-agreement (RBA)} algorithm from \citet{dynamic-routing}, which uses predictions from lower-level capsules to generate upper-level capsules. Later, \citet{em-routing} introduced an improved expectation-maximization based routing algorithm called \emph{EM-routing} that uses a pose-matrix to calculate equivariant upper-level capsules.

Although \citet{dynamic-routing} and \citet{em-routing} already successfully demonstrated this innovative idea to connect capsules of different layers with an iterative routing procedure, we were able to improve on both algorithms. More precisely, we found that routing-by-agreement and EM-routing prevent capsule networks from being universal approximators because they limit their representation capabilities to symmetric functions. Therefore, removing this limitation increases the expressivity and improves capsule networks as we will show in detail later in this paper. 

The existence of this limitation can not only be evaluated empirically, we also prove theoretically that for every input to a capsule layer a negative input exists that cannot be distinguished by capsules, even when this input represents a different class. Our theoretical investigation is independent of specific hyperparameters and its independent with respect to the used architecture. Therefore, the presented proof holds for the complete hyperparameter-space. We also use this proof to find specific problems that can not be solved with capsule networks using RBA or EM-routing. One such problem is the classification of the $\sign$ of scalar inputs, which we will show empirically later in this paper.

We also found that this limitation negatively affects the training of deep capsule networks for arbitrary datasets. This is especially problematic for deep networks since the first primary capsules are constructed with ReLU convolutions which means that all components of primary capsules are positive. Therefore, the negative of a vector and the vector itself can never both be an input to the first capsule layer and the limitation mentioned above will not apply to the first layer. After the first capsule layer, routing algorithms are used to generate subsequent capsules and the negative input can indeed occur in practice. Each subsequent layer increases the probability that negative inputs occur, and we hypothesize that those routing algorithms impact the training of deep architectures negatively. This hypothesis is extensively evaluated in the experimental section.

Additionally, we will present a solution to circumvent this limitation by utilizing a bias term, solving the aforementioned weakness and further improving state-of-the-art routing algorithms. Neural networks usually add such a term to the sum of all weighted inputs, but bias terms are not mentioned by \citet{dynamic-routing} for RBA in their paper. Some authors, such as \citet{capsule-rotation-traffic-sign}, explicitly state that for capsule networks no bias is used and that \emph{"[...] this design of the capsule allows more capabilities in representing its features"}. We agree that a bias term is not mentioned in the paper by \citet{dynamic-routing}, but prove that this limits the expressivity of capsule networks. On the other hand, for EM-routing \citet{em-routing} describes that they use two learned biases per capsule, namely $\beta_a$ and $\beta_u$. We prove in this paper that these terms are not enough to remove the aforementioned limitation and will present a solution to this problem.

The novel contributions of this paper are:
\begin{itemize}
    \item We proof that RBA as well as EM-routing decrease the expressivity of capsule networks, preventing them from being universal approximators when RBA or EM-routing are used.
    \item We shown that this limitation affects the training of deep capsule networks negatively.
    \item We adapt both algorithms to remove the aforementioned weakness and show that this adaption (1) stabilizes the training, (2) enables deeper capsule networks and, (3) increases the accuracy.
\end{itemize}

In the next section we will discuss related work. In \Cref{sec:proof} we will prove that a capsule network cannot distinguish inputs and the negation of those inputs if routing-by-agreement or EM-routing is used to connect different capsule layers. A bias term that we introduce in \Cref{sec:methods} removes this limitation. In the experimental \Cref{sec:experiments} we show that the $\sign$ function is one specific function that can not be expressed with capsule networks and we show that the proposed bias term removes this limitation. In the same section, we also show that a bias term stabilizes the training of capsule networks, enables deeper capsule networks to train successfully and increases the achievable accuracy. In \Cref{sec:conclusion} we discuss our findings.

%%%%%%%%%%%%%%%%%%%%%%%%%%%%%%%%%%%%%%%%
\section{Related Work}\label{sec:related_work}
Capsules were introduced by \citet{transforming-autoencoders} who showed that such capsules can be trained by backpropagating the difference between the actual and the target outputs. To implement this idea, \citet{dynamic-routing} introduced the concept of routing algorithms. The goal of these algorithms is to activate upper-level capsules from lower-level capsules in such a way that during inference the activated capsules form a parse-tree that represents the decomposition of objects into their parts. The \emph{RBA} algorithm \citep{dynamic-routing} uses vectors to activate such upper-level capsules. The orientation of these vectors encodes the instantiation parameters whereas the activation of a capsule is encoded by the norm of the vector. The \emph{CapsNet} architecture showed the effectiveness of this vector-based algorithm. Later, \citet{em-routing} introduced an \emph{expectation maximization} based routing algorithm (\emph{EM-routing}) to route activations from a lower-level layer to upper-level layers. Compared to the CapsNet architecture, the authors use so-called matrix-based capsules rather than vector-based capsules. A \emph{matrix capsule} is composed of a $4 \times 4$ pose matrix together with a scalar variable that represents the activation of this capsule. An RBA based routing algorithm that uses scaled distance agreements rather than dot products to calculate general agreements is introduced by \citet{gamma-capsule-networks}. Another algorithm from \citet{inverted-dot-product-routing} uses inverted attentions to connect capsules of different layers. In this work, higher-level parents compete for the attention of lower-level child capsules instead of the other way round. \citet{stacked-capsule-autoencoders} proposes an unsupervised version of capsules to segment images into constituent parts and to organize discovered parts into objects. \citet{deepcaps} created some deeper capsule networks, resulting in state-of-the-art performance on MNIST, SVHN, smallNORB, and fashionMNIST. The routing strategy is formulated as an optimization problem that minimizes a KL regularization term between the current coupling and its state by \citet{optimization-view-rba}. The popularity of capsule networks has increased and they are used in a wide range of applications such as lung cancer screening \citep{lung-cancer-screening}, action detection from video \citep{video-capsnet}, or object classification in 3D point clouds \citep{3d-point-capsules}, to name a few.

%%%%%%%%%%%%%%%%%%%%%%%%%%%%%%%%%%%%%%%%
\section{Limitation of capsule networks} \label{sec:proof}
In this section, we prove that a capsule network, using the routing-by-agreement as 
well as the EM-routing algorithm, cannot distinguish inputs and their negative inputs.

\subsection{Notation}
$I$ represents the number of lower-level capsules and $J$ the number of upper-level capsules. We use $i \in 1, ..., I$ to denote a single lower-level capsule and $j \in 1, ..., J$ to denote an upper-level capsules. We call the input $u$ for RBA the positive input and $-u$ the negative input vector. Note that the input for EM-routing is $M$. For any variable $v$ in the algorithm, we refer to the variable that is calculated using $u$ with $v^+$ and using $-u$ with $v^-$. With $v_{ij}$ we reference a variable that is used for lower-level capsule $i$ and upper-level capsule $j$. To reference component $h$ of $v \in \mathbb{R}^H$ we write $v^h$. For example with $v_{ij}^{-h}$ we access component $h$ of vector $v \in \mathbb{R}^H$ from lower-level capsule $i \in 1, ..., I$ to upper-level capsule $j \in 1, ..., J$ that is calculated for the negative input.

%%%%%%%%%%%%%%%%%%%%%%%%%%%%%%%%%%%%%%%%%%%%%%%
\subsection{Routing-by-agreement}

\begin{algorithm}[t]
  \caption{Details for the routing-by-agreement algorithm are described by \citet{dynamic-routing}.  \newline $\forall$ capsules $i$ in layer $l$ and $j$ in layer $l+1$ with $r$ routing iterations and predictions $\hat{u}_{j|i}$. The output of the algorithm are instantiation vectors $v_j$.} \label{algo:routing-by-agreement}
  \begin{algorithmic}[1]
    \Procedure{RoutingByAgreement}{$\hat{u}_{j|i},r,l$}
    \State $\forall b_{ij} , b_{ij} \gets 0$ \label{algo:rba:line:initb}
    \For{$r \text{ iterations}$}
    \State $c_{ij} \gets \frac{\exp(b_{ij})}{\sum_l \exp(b_{il})}$ \label{algo:rba:line:cij} 
    \State $s_j \gets \sum_ic_{ij}\hat{u}_{j|i}$  
    \State $v_j \gets \frac{||s_j||^2}{1+||s_j||^2} \frac{s_j}{||s_j||}$
    \State $b_{ij} \gets b_{ij} + v_j \cdot \hat{u}_{j|i}$\label{algo:rba:line:bij}
    \EndFor
    \EndProcedure
  \end{algorithmic}
\end{algorithm}

We prove here that for all output capsules the activations for inputs and corresponding negative inputs are equal if the routing-by-agreement algorithm is used. The algorithm and its parameters are shown in \cref{algo:routing-by-agreement}. The inputs for this routing algorithm are lower-level predictions $\hat{u}_{j|i} = W_{ij} u_i$ where $u_i$ is the activation of lower-level capsule $i$ and $W_{ij}$ the transformation matrix that is learned through backpropagation.

\begin{lemma}\label{lm:rba_prediction}
  The prediction vector $\hat{u}^+_{j|i}$ is negated
  ($\hat{u}^-_{j|i} = -\hat{u}^+_{j|i}$) when the input $u_i$ is negated.
\end{lemma}
\begin{proof}
  \begin{align*}
  \hat{u}^-_{j|i} = W_{ij} (-u_i) = - W_{ij} u_i = -\hat{u}^+_{j|i}
  \end{align*}
\end{proof}

\begin{lemma}\label{lm:rba_negative_activation}
  Assume that $c_{ij}^+ = c_{ij}^-$, then the activation $v_j^+$ of a capsule $j$ can be negated ($v_j^- = -v_j^+$) by negating all inputs ${u}_{1}, {u}_{2}, ... {u}_{I}$.
\end{lemma}
\begin{proof}
  For the preactivation $s_j = \sum_ic_{ij}\hat{u}_{j|i}$ and the negated inputs $u_i$:
  \begin{align*}
    s^-_j &= \sum_i c^-_{ij} \hat{u}^-_{j|i}  && \text{Definition of $s^-_j$} \\
    &= \sum_i c^+_{ij} (-\hat{u}^+_{j|i}) && \text{Assumption, \cref{lm:rba_prediction}} \\
    &= - \sum_i c^+_{ij} \hat{u}^+_{j|i} = -s^+_j   && \text{Definition of $s^+_j$}
  \end{align*}
  and therefore
  \begin{align*}
    v^-_j &= \frac{||s^-_j||^2}{1+||s^-_j||^2} \frac{s^-_j}{||s^-_j||} && \text{Definition of $v^-_j$} \\
    &= \frac{||-s^+_j||^2}{1+||-s^+_j||^2} \frac{-s^+_j}{||-s^+_j||}  \\
    &= -\frac{||s^+_j||^2}{1+||s^+_j||^2} \frac{s^+_j}{||s^+_j||} = -v^+_j && \text{Definition of $v^+_j$} \\
  \end{align*}
\end{proof}

We see that the activation of a capsule can be negated by negating all inputs, under the assumption that $c^+_{ij} = c^-_{ij}$. We will now show that the property $c^+_{ij} = c^-_{ij}$ holds in every iteration.

\begin{lemma}\label{lm:rba_stable_bij}
  The routing-by-agreement algorithm produces coupling coefficients 
  $c^+_{ij} = c^-_{ij}$ in 
  every routing iteration.
\end{lemma}
\begin{proof}
  First we will show that $b^{+}_{ij} = b^{-}_{ij}$ for any routing 
  iteration, using a proof by induction on the routing iterations.
    \newline
  \emph{Base case (BC)}: For the first routing iteration and the initialization of 
  $b^{+}_{ij} = b^{-}_{ij} = 0$, we can show that
  \begin{align*}
    c^{+}_{ij} 
    = \frac{1}{\sum_l \exp(0)} 
    = c^{-}_{ij}
  \end{align*}
  Therefore, \cref{lm:rba_negative_activation} is applicable which implies that 
  \begin{align*}
    b^{-}_{ij} &= 0 + \left(v^{-}_j \cdot  \hat{u}^{-}_{j|i}\right)              && \text{Definition of $b^{-}_{ij}$}    \\
                &= 0 + \left(-v^{+}_j \cdot -\hat{u}^{+}_{j|i}\right)             && \text{\Cref{lm:rba_prediction},  \cref{lm:rba_negative_activation}} \\
                &= 0 + \left(v^{+}_j \cdot \hat{u}^{+}_{j|i}\right) = b^{+}_{ij}
  \end{align*}
  \newline
  \emph{Inductive hypothesis (IH)}: $b^{-}_{ij} = b^{+}_{ij}$ in the previous iteration.
\newline
  \emph{Inductive step (IS)}: We now show that $b^{-}_{ij} = b^{+}_{ij}$ also holds in the next iteration. At first \cref{lm:rba_negative_activation} is 
  applicable, because $b_{ij}^- = b_{ij}^+$ also holds in the previous iteration. So we can show that
  \begin{align*}
    c^{-}_{ij} 
    = \frac{\exp(b^{-}_{ij})}{\sum_l \exp(b^{-}_{il})} 
    = \frac{\exp(b^{+}_{ij})}{\sum_l \exp(b^{+}_{il})} 
    = c^{k+}_{ij}
  \end{align*}
  Therefore we can prove for $b_{ij}$ in the next iteration that
  \begin{align*}
    b^{-}_{ij} &= b^{-}_{ij} + \left(v^{-}_j \cdot \hat{u}^-_{j|i}\right)  && \text{Definition of $b^{-}_{ij}$}  \\
                &= b^{+}_{ij} + \left(-v^{+}_j \cdot -\hat{u}^+_{j|i}\right)   && \text{\Cref{lm:rba_prediction},  \cref{lm:rba_negative_activation}, IH} \\
                &= b^{+}_{ij} + \left(v^{+}_j \cdot \hat{u}^+_{j|i}\right) = b^{+}_{ij} \\
  \end{align*}

  This proof by induction shows that in every iteration the property 
  $b^{+}_{ij} = b^{-}_{ij}$ holds. Therefore also for the coupling coefficients
  it holds that
  \begin{align*}
    c^{-}_{ij} 
    = \frac{\exp(b^{-}_{ij})}{\sum_l \exp(b^{-}_{il})}
    = \frac{\exp(b^{+}_{ij})}{\sum_l \exp(b^{+}_{il})} 
    = c^{+}_{ij} 
  \end{align*}
\end{proof}

\begin{lemma}\label{lm:rba_arbitrary_inputs}
  For arbitrary inputs, the output of a capsule layer 
  is negated whenever all inputs are negated.
\end{lemma}
\begin{proof}
  \Cref{lm:rba_negative_activation} shows that the activation vector can be negated 
  by negating all inputs under the assumption that $c^-_{ij} = c^+_{ij}$. 
  \Cref{lm:rba_stable_bij} shows that $c^{-}_{ij} = c^{+}_{ij}$ at any iteration. Therefore \cref{lm:rba_negative_activation} is always applicable and outputs are negated whenever all inputs are negated.
\end{proof}

\begin{theorem}
  A capsule network with RBA cannot distinguish inputs and their negative inputs.
\end{theorem}
\begin{proof}
  By recursively applying \cref{lm:rba_arbitrary_inputs} at every layer we see that a negated input produces a negated output at the last layer such that $v^-_j = -v^+_j$. The classification of an input is based on the norm of the output vector. But $||v^-_j|| = ||v^+_j||$ holds and therefore a capsule network cannot 
  distinguish the input and the negated input.
\end{proof}

\subsection{EM-routing}

\begin{algorithm}[t]
  \caption{Details for the routing-by-agreement algorithm are described by \citet{em-routing}.  \newline $\forall$ capsules $i$ in layer $l$ and $j$ in layer $l+1$ with $r$ routing iterations and predictions $\hat{u}_{j|i}$. The values for $\beta_a$ and $\beta_u$ are learned through backpropagation. The output of the algorithm are the pose matrices $V_j$ and the activations $a_j$.} \label{algo:em-routing}
  \begin{algorithmic}[1]
    \Procedure {EM Routing:}{$a, V$}
      \State $R_{ij} \gets 1 / J$
      \For{$r \text{ iterations}$}
        \State $\forall j:$ M-Step$(a, R, V, j)$
        \State $\forall i:$ E-Step$(\mu, \sigma, a, V, i)$
      \EndFor
      \State {\bfseries return } $a$, $M$
    \EndProcedure
    \Procedure {M-Step:}{$a, R, V, j$}
      \State \ \ \ \ $\forall i: R_{ij} \gets R_{ij} * a_i$
      \State \ \ \ \ $\forall h: \mu_j^h \gets \frac{\sum_i R_{ij} V^h_{ij}}{\sum_i R_{ij}}$
      \State \ \ \ \ $\forall h: (\sigma_j^h)^2 \gets \frac{\sum_i R_{ij} (V_{ij}^h - \mu_j^h)^2}{\sum_i R_{ij}}$
      \State \ \ \ \ $\cost^h \gets (\beta_u + \log \sigma_j^h) \sum_i R_{ij}$
      \State \ \ \ \ $a_j \gets \logistic\left(\lambda (\beta_a - \sum_h^H \cost^{h})\right)$
    \EndProcedure
    \Procedure{E-Step:}{$\mu, \sigma, a, V, i$}
      \State \ \ \ \ $\forall j: p_j \gets \frac{\exp\left( -\sum^H_h \frac{(V_{ij}^{h} - \mu_j^{h})^2}{2(\sigma_j^{h})^2} \right)}{\sqrt{\prod_h^H 2 \pi (\sigma_j^{h})^2}}$
      \State \ \ \ \ $\forall j: R_{ij} \gets \frac{a_j p_j}{\sum_k^J a_k p_k}$
    \EndProcedure
  \end{algorithmic}
\end{algorithm}

In this subsection we prove that for all output capsules, the activations for inputs and their negated inputs are equal if EM-routing is used. The algorithm and its parameters are shown in \cref{algo:em-routing}. Votes from a lower-level capsule $i$ to an upper-level capsule $j$ are calculated with $V_{ij} = M_i W_{ij}$ for the lower-level pose matrix $M_i$ and the transformation matrix $W_{ij}$ which is learned through backpropagation. Note also that $\mu_j$ is a vectorized version of the pose matrix $M_j$.

First, we prove some lemmas that will be used under the assumption that $R_{ij}^+ = R_{ij}^-$. We will then show that this assumption holds after every routing iteration and conclude our proof.

\begin{lemma}\label{lm:em_prediction}
  Votes $V_{ij}^+$ are negated when the input $M_{ij}^+$ is negated.
\end{lemma}
\begin{proof}
  \begin{align*}
    V_{ij}^+ = M_{ij}^+ W_{ij} = -M_{ij}^- W_{ij} = -V_{ij}^-
  \end{align*}
\end{proof}

\begin{lemma}\label{lm:em_mean}
  Assume that $R_{ij}^+ = R_{ij}^-$, then $\mu^{+h}_j = -\mu_j^{-h}$
\end{lemma}
\begin{proof}
  \begin{align*}
    \mu_j^{+h} &= \frac{\sum_i^I R_{ij}^+ V_{ij}^{+h}}{\sum_i^I R_{ij}^+} && \text{Definition of $\mu_j^{+h}$} \\ 
    &= -\frac{\sum_i^I R_{ij}^+ V_{ij}^{-h}}{\sum_i^I R_{ij}^+} = -\mu_j^{-h} && \text{\Cref{lm:em_prediction}} 
  \end{align*}
\end{proof}

\begin{lemma}\label{lm:em_variance}
  Assume that $R_{ij}^+ = R_{ij}^-$, then $(\sigma_j^{+h})^2 = (\sigma_j^{-h})^2$
\end{lemma}
\begin{proof}
  \begin{align*}
    (\sigma_j^{+h})^2 &= \frac{\sum_i^I R_{ij}^+ (V_{ij}^+ - \mu^{+h}_j)^2}{\sum_i^I R_{ij}^+} && \text{Definition of $\sigma_j^{+h})^2$} \\
    &= \frac{\sum_i^I R_{ij}^- (V_{ij}^+ - \mu^{+h}_j)^2}{\sum_i^I R_{ij}^-} && \text{Assumption} \\
    &= \frac{\sum_i^I R_{ij}^- (-1 (V_{ij}^- - \mu^{-h}_j))^2}{\sum_i^I R_{ij}^-} && \text{\Cref{lm:em_prediction}, \cref{lm:em_mean}} \\
    &= (\sigma_j^{-h})^2
  \end{align*}
\end{proof}

\begin{lemma}\label{lm:em_cost}
  Assume that $R_{ij}^+ = R_{ij}^-$, then $\cost^{+h} = \cost^{-h}$
\end{lemma}
\begin{proof}
  \begin{align*}
    \cost{}^{+h} &= \left( \beta_u + \log(\sigma_j^{+h}) \right) \sum_i^I R_{ij}^+  && \text{Definition of $\cost{}^{+h}$} \\
    &= \left( \beta_u + \log(\sigma_j^{-h}) \right) \sum_i^I R_{ij}^-
    = \cost{}^{-h} && \text{\Cref{lm:em_variance}}
  \end{align*}
\end{proof}

\begin{lemma}\label{lm:em_a}
  Assume that $R_{ij}^+ = R_{ij}^-$, then $a^{+}_j = a^{-}_j$
\end{lemma}
\begin{proof}
  \begin{align*}
    a^+_j &= \logistic\left(\lambda (\beta_a - \sum_h^H \cost{}^{+h})\right) && \text{Definition of $a^+_j$} \\
    &= \logistic\left(\lambda (\beta_a - \sum_h^H \cost{}^{-h})\right)
    = a^{-}_j && \text{\Cref{lm:em_cost}}
  \end{align*}
\end{proof}

We will now show that the assumption $R_{ij}^+ = R_{ij}^-$ holds in every iteration of EM-routing:
\begin{lemma}\label{lm:em_stable_r}
  The EM routing algorithm, with inputs
  $a_i, \beta_u, \beta_a, \lambda$, produces coupling coefficients 
  $R^{+}_{ij} = R^{-}_{ij}$ in 
  every routing iteration
\end{lemma}
\begin{proof}
  We prove this by induction on the routing iterations.
    \newline
  \emph{Base case (BC):} $R_{ij}^+ = R_{ij}^-$ in the first routing iteration
  \begin{align*}
    R_{ij}^+ = \frac{a_i}{J} = R_{ij}^{-}
  \end{align*}
    \newline
  \emph{Inductive hypothesis (IH):}
  $R_{ij}^{+} = R_{ij}^{-}$ in the previous iteration.
    \newline
  \emph{Inductive step (IS):} We know show that $R_{ij}^{+} = R_{ij}^{-}$ holds also in the next routing iteration.
  
  Using the inductive hypothesis and \cref{lm:em_mean}, \cref{lm:em_variance}, \cref{lm:em_a}  
  it follows that $\mu^{+h}_j = -\mu_j^{+h}$, $(\sigma_j^{+h})^2 = (\sigma_j^{-h})^2$
  and $a^{+h} = a^{-h}$. Therefore it holds at 
  \begin{align*}
    p_j^{+} &= \frac{\exp\left( -\sum^H_h \frac{(V_{ij}^{+h} - \mu_j^{+h})^2}{2(\sigma_j^{+h})^2} \right)}{\sqrt{\prod_h^H 2 \pi (\sigma_j^{+h})^2}} \\ 
    &= \frac{\exp\left( -\sum^H_h \frac{(V_{ij}^{-h} - \mu_j^{-h})^2}{2(\sigma_j^{-h})^2} \right)}{\sqrt{\prod_h^H 2 \pi (\sigma_j^{-h})^2}}
    = p_j^{-}
  \end{align*}
  and we conclude that 
  \begin{align*}
    R_{ij}^{+} &= \frac{a_j^{+}p_j^{+}}{\sum_{t}^J a_t^{+}P_t^{+}}
    = \frac{a_j^{-}p_j^{-}}{\sum_{t}^J a_t^{-}P_t^{-}} = R_{ij}^{-}
  \end{align*}
\end{proof}

We have seen that $R_{ij}^{+} = R_{ij}^{-}$ is true in every routing iteration. Using the previously proven lemmas and $R_{ij}^{+} = R_{ij}^{-}$ we can conclude that $a_j^+ = a_j^-$:

\begin{lemma}
  For arbitrary inputs to a capsule layer, the output activation $a_j$ for given inputs and negated inputs does not change. The pose matrix $M_j$ is negated for negated inputs.
\end{lemma}
\begin{proof}
  \Cref{lm:em_stable_r} shows that the routing coefficient does not change if inputs are negated. From \cref{lm:em_a} it follows that the activation does not change and from \cref{lm:em_mean} it follows that the pose matrix $M_j$ is negated.
\end{proof}

\begin{theorem}
  A capsule network with EM-routing cannot distinguish inputs and their negated inputs.
\end{theorem}
\begin{proof}
  By recursively applying the preceding lemma at every layer we see that the negated input pose matrix $-M$ and the input activation $a$ produces a negated output pose matrix, but the same output activation for an input and its negation.
\end{proof}

%%%%%%%%%%%%%%%%%%%%%%%%%%%%%%%%%%%%%%%%%%%%%%%
\section{Adding a bias to RBA and EM-Routing}\label{sec:methods}
In this section, we propose two simple ways to avoid the weaknesses of the routing algorithms proven in the previous sections. To accomplish this, we add a bias term to the preactivations of the RBA algorithm and we add a bias to the pose matrix of EM-routing. 

\subsection{Routing-by-Agreement}
To solve this limitation we follow the original implementation of \citet{dynamic-routing} as mentioned in \Cref{sec:introduction}. This modification targets \cref{lm:rba_negative_activation} so that negative activation vectors cannot preserve their norm. A different norm ensures that the classification is also different. If we add biases to the calculation of the preactivation as follows  

\begin{align*}
  s_j = \left(\sum_i c_{ij} \hat{u}^-_{j|i} \right) + bias_j
\end{align*}

Then, a network can learn a $bias_j \neq 0$ such that:

\begin{align*}
  s^-_j &= \left(\sum_i c^-_{ij} \hat{u}^-_{j|i} \right) + bias_j \\
        &= \left( \sum_i c^-_{ij} (-\hat{u}^+_{j|i}) \right) + bias_j
        = - \left(\sum_i c^+_{ij} \hat{u}^+_{j|i} \right) + bias_j \\
        &\neq - \left(\left(\sum_i c^+_{ij} \hat{u}^+_{j|i} \right) + bias_j\right) = -s^+_j \\
\end{align*}
This enables the network to learn non zero bias parameters such that the calculated preactivation vector is different for inputs and their negative inputs. This leads to the activation vectors and the norm of the activation vectors being different. Therefore, the network can learn a bias so that inputs and negated inputs can be distinguished.

\subsection{EM-Routing}
For EM-routing we proceed similarly, but the target is \cref{lm:em_prediction} to let the network learn through a bias to avoid that inputs and their negations cannot be distinguished. We update the calculation of the votes with:
\begin{align*}
  V_{ij} = M_{ij} W_{ij} + \bias_j
\end{align*}

A network can learn a $bias_j \neq 0$ such that:
\begin{align*}
  V_{ij}^+ &= M_{ij}^+ W_{ij} + \bias_j = -(M_{ij}^- W_{ij}) + \bias_j \\ 
  &\neq -\left((M_{ij}^- W_{ij}) + \bias_j\right) = V_{ij}^-
\end{align*}
and therefore \cref{lm:em_a} does not hold so that the activation for inputs and their negations are different. It is important to mention that this change ensures that only the votes of lower-level capsules are changed and not the activations $a_i$.

%%%%%%%%%%%%%%%%%%%%%%%%%%%%%%%%%%%%%%%%%%%%%%%
\section{Experimental evaluation}\label{sec:experiments}

\subsection{Setup}
Our implementation of the capsule network uses TensorFlow 2.3 from \citet{tensorflow} and is available on GitHub\footnote{\url{https://github.com/peerdavid/capsnet-limitations}}. To be able to compare both routing algorithms we use RBA as well as EM-routing following the CapsNet architecture from \citet{dynamic-routing}. To avoid numerical issues of EM-routing during training we use the adapted version of EM-routing as introduced by \citet{em-routing-pitfalls}. The margin loss from \citet{dynamic-routing} is minimized using the Adam optimizer from \citet{adam} with a learning rate of $10^{-4}$ and a batch size of $128$. We generated a new dataset as described in experiment 1 and used MNIST from \citet{mnist}, fashionMNIST from \citet{fashion-mnist}, SVHN from \citet{svhn} and smallNORB from \citet{smallnorb} for experiment 2. The test accuracy is evaluated using the test set provided by each dataset.

\subsection{Classifying the $sign$ of a scalar variable}
\begin{figure}[t]
  \centering
  \includegraphics[width=4.0cm]{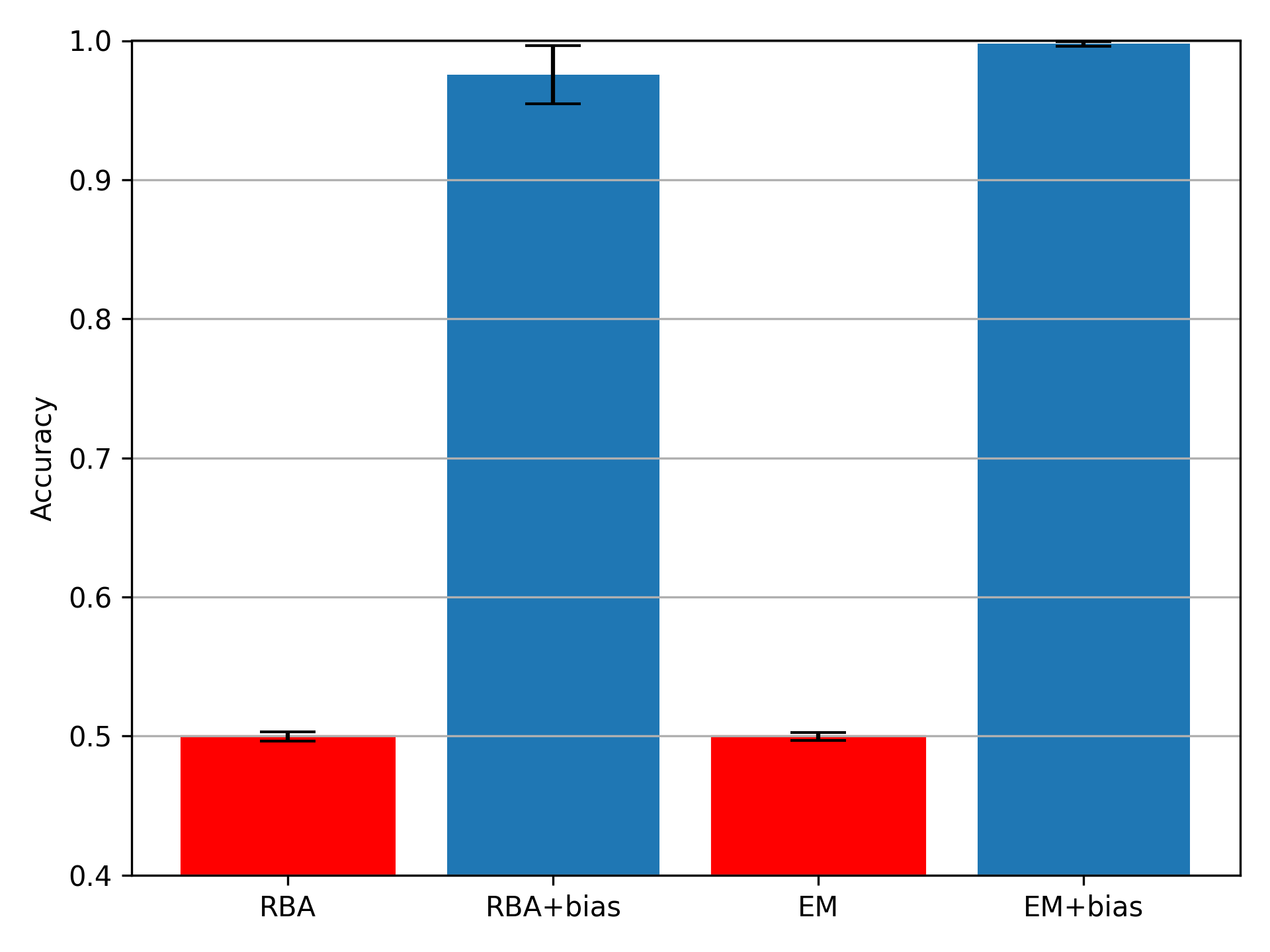}
  \caption{Mean training accuracy and the std. for $1200$ different capsule networks trained using RBA and EM-routing with and without a bias term.}\label{fig:sign_results}
\end{figure}

In \Cref{sec:proof} we showed that a capsule network can not distinguish inputs and its negative counterpart which we will evaluated empirically in this experiment. To this end, we generate a dataset with $20$k samples, where each input and its negative counterpart represent different classes: $x_i \in [-1, +1]$ where $x_i$ of class $1$ iff $x_i < 0$ and class $2$ otherwise. To evaluate the correctness of the proof from \Cref{sec:proof} empirically, we measure if the training accuracy is not better than chance for both, RBA and EM routing. On the other hand, if this limitation is avoided by using the method introduced in \Cref{sec:methods}, the problem should be easily solvable. As stated, the proof is independent with respect to the used architecture. To demonstrate this, we evaluate many different architectures in this experiment: The inputs are fed into a network with $1$ up to $4$ capsule layers. For each layer, we vary the number of capsules between $10$ and $30$ in steps of $5$ and we also vary the dimensionality of each capsule between $10$ and $18$ in steps of two. Each network is trained $3$ times with different random initialized weights for RBA, RBA+bias, EM-routing, and EM-routing + bias resulting in a total of $1200$ trained networks.

The mean accuracy together with the std. grouped by routing algorithms are reported in \cref{fig:sign_results}. 
\emph{Capsule networks without a bias term are not able to achieve an accuracy higher than chance whereas models that use a bias term can solve this task with high accuracy}, which supports our hypothesis. We would also like to mention that not only the mean accuracy is low for capsule networks without a bias term. We evaluated each model and none were able to achieve an accuracy above chance. This demonstrates the importance of a bias term and supports the correctness of the proof from \Cref{sec:proof} empirically. We also hypothesized in \Cref{sec:introduction} that this limitation affects the training of deep capsule networks negatively which we will evaluate in the next experiment.

\begin{table*}[t]
  \caption{Mean and std. test accuracy (3 runs) for a capsule network with RBA for $n$ layers.}
  \label{tbl:deep_rba}
  \centering
  \begin{small}
  \begin{sc}
  \begin{tabular}{ll|ccccc}
    \hline\noalign{\smallskip}
    Dataset  & Method     & $2$   & $3$   & $4$   & $5$ & $6$ \\
    \noalign{\smallskip} \hline \noalign{\smallskip}
    mnist   & rba       & \boldmath$99.4 \pm 0.03$ & \boldmath$99.5 \pm 0.03$ & $99.4 \pm 0.02$ & $9.8 \pm 0.00$ & $9.8 \pm 0.00$ \\
            & \bc rba+bias  & \boldmath$99.4 \pm 0.01$ & \boldmath$99.5 \pm 0.02$ & \boldmath$99.5 \pm 0.04$ & \boldmath$99.3 \pm 0.06$ & \boldmath$99.0 \pm 0.11$ \\
    fashion & rba           & $87.4 \pm 0.20$ & $89.0 \pm 0.06$ & \boldmath$89.0 \pm 0.16$ & $10.0 \pm 0.00$ & $10.0 \pm 0.00$  \\
            & \bc rba+bias  & \boldmath$87.5 \pm 0.26$ & \boldmath$89.1 \pm 0.08$ & \boldmath$89.0 \pm 0.08$ & \boldmath$87.9 \pm 0.19$ & \boldmath$85.3 \pm 1.71$ \\
    svhn    & rba           & \boldmath$91.9 \pm 0.10$ & $92.9 \pm 0.16$ & $92.3 \pm 0.05$ & $6.7 \pm 0.00$ & $6.7 \pm 0.00$ \\
            & \bc rba+bias  & \boldmath$91.9 \pm 0.15$ & \boldmath$93.0 \pm 0.14$ & \boldmath$92.5 \pm 0.24$ & \boldmath$91.4 \pm 0.07$ & \boldmath$87.1 \pm 0.79$ \\
    norb    & rba           & \boldmath$91.0 \pm 0.24$ & $89.4 \pm 0.48$ & $88.2 \pm 0.53$ & $20.0 \pm 0.00$ & $20.0 \pm 0.00$ \\
            & \bc rba+bias  & \boldmath$91.0 \pm 0.11$ & \boldmath$89.7 \pm 0.19$ & \boldmath$88.8 \pm 0.39$ & \boldmath$87.1 \pm 0.32$ & \boldmath$82.5 \pm 1.46$ \\
    \hline
  \end{tabular}
    \end{sc}
  \end{small}
\end{table*}

\begin{table*}[t]
  \caption{Mean and std. test accuracy (3 runs) for a capsule network with EM-routing for $n$ layers.}
  \label{tbl:deep_em}
  \centering
  \begin{small}
  \begin{sc}
  \begin{tabular}{ll|ccccc}
    \hline\noalign{\smallskip}
    Dataset & Method     & $7$   & $8$   & $9$   & $10$ & $11$   \\
    \noalign{\smallskip} \hline \noalign{\smallskip}
    mnist   & em           & $99.3 \pm 0.04$ & $99.2 \pm 0.06$ & $98.8 \pm 0.20$ & $35.3 \pm 27.62$ & $10.7 \pm 0.26$ \\
            & \bc em+bias  & \boldmath$99.4 \pm 0.03$ & \boldmath$99.3 \pm 0.06$ & \boldmath$99.2 \pm 0.07$ & \boldmath$99.1 \pm 0.04$ & \boldmath$99.0 \pm 0.04$ \\
    fashion & em           & $87.8 \pm 0.18$ & $87.1 \pm 0.39$ & $83.9 \pm 0.64$ & $37.7 \pm 21.54$ & $10.7 \pm 0.07$ \\
            & \bc em+bias  & \boldmath$88.3 \pm 0.05$ & \boldmath$87.9 \pm 0.26$ & \boldmath$87.8 \pm 0.23$ & \boldmath$87.4 \pm 0.26$ & \boldmath$87.0 \pm 0.10$ \\
    svhn    & em           & \boldmath$90.4 \pm 0.19$ & $85.3 \pm 1.83$ & $50.1 \pm 11.68$ & $24.0 \pm 6.23$ & $19.6 \pm 0.01$ \\
            & \bc em+bias  & $89.5 \pm 0.26$ & \boldmath$89.2 \pm 0.34$ & \boldmath$88.5 \pm 0.26$ & \boldmath$88.3 \pm 0.14$ & \boldmath$87.7 \pm 0.40$ \\
    norb    & em        & \boldmath$88.0 \pm 0.34$ & $74.8 \pm 2.03$ & $20.8 \pm 0.18$ & $20.5 \pm 0.11$ & $20.5 \pm 0.04$ \\   
            & \bc em+bias  & $87.6 \pm 0.73$ & \boldmath$87.8 \pm 0.85$ & \boldmath$87.0 \pm 0.65$ & \boldmath$86.2 \pm 0.48$ & \boldmath$86.0 \pm 0.28$ \\
    \hline
  \end{tabular}
  \end{sc}
  \end{small}
\end{table*}

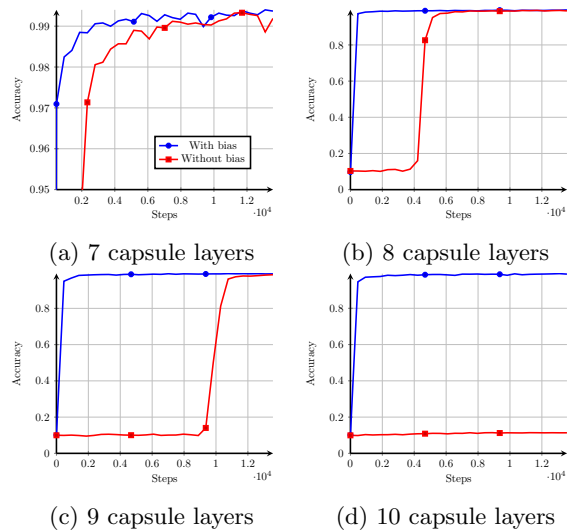
\begin{figure}[htbp]
  \centering
  \begin{subfigure}[b]{0.23\textwidth}
    \begin{tikzpicture}[scale=0.42]
      \begin{axis}[xmajorgrids, ymajorgrids, axis lines=left, mark repeat=10, line width=0.5mm, ymin=0.95, legend style={at={(0.90, 0.3)}}, ylabel={Accuracy}, xlabel={Steps}]
      \addplot table [x=Step, y=Value, col sep=comma] {data/bias_7.csv};
      \addlegendentry{With bias};
    
      \addplot table [x=Step, y=Value, col sep=comma] {data/no_bias_7.csv};
      \addlegendentry{Without bias};
      \end{axis}
    \end{tikzpicture}
    \caption{$7$ capsule layers}
    \label{fig:em_curve_7}
  \end{subfigure}
  \begin{subfigure}[b]{0.23\textwidth}
    \begin{tikzpicture}[scale=0.42]
      \begin{axis}[xmajorgrids, ymajorgrids, axis lines=left, mark repeat=10, line width=0.5mm, ymin=0.0, legend style={at={(0.90, 0.3)}}, ylabel={Accuracy}, xlabel={Steps}]
      \addplot table [x=Step, y=Value, col sep=comma] {data/bias_8.csv};
      \addplot table [x=Step, y=Value, col sep=comma] {data/no_bias_8.csv};
      \end{axis}
    \end{tikzpicture}
    \caption{$8$ capsule layers}
    \label{fig:em_curve_8}
  \end{subfigure}
  \begin{subfigure}[b]{0.23\textwidth}
    \begin{tikzpicture}[scale=0.42]
      \begin{axis}[xmajorgrids, ymajorgrids, axis lines=left, mark repeat=10, line width=0.5mm, ymin=0.0, legend style={at={(0.90, 0.3)}}, ylabel={Accuracy}, xlabel={Steps}]
      \addplot table [x=Step, y=Value, col sep=comma] {data/bias_9.csv};
      \addplot table [x=Step, y=Value, col sep=comma] {data/no_bias_9.csv};
      \end{axis}
    \end{tikzpicture}
    \caption{$9$ capsule layers}
    \label{fig:em_curve_9}
  \end{subfigure}
  \begin{subfigure}[b]{0.23\textwidth}
    \begin{tikzpicture}[scale=0.42]
      \begin{axis}[xmajorgrids, ymajorgrids, axis lines=left, mark repeat=10, line width=0.5mm, ymin=0.0, legend style={at={(0.90, 0.3)}}, ylabel={Accuracy}, xlabel={Steps}]
      \addplot table [x=Step, y=Value, col sep=comma] {data/bias_10.csv};
      \addplot table [x=Step, y=Value, col sep=comma] {data/no_bias_10.csv};
      \end{axis}
    \end{tikzpicture}
    \caption{$10$ capsule layers}
    \label{fig:em_curve_10}
  \end{subfigure}

  \caption{Comparison of the training curve for a capsule networks that uses EM-routing with and without a bias term.}
  \label{fig:em_curve}
\end{figure}

\subsection{Training deep capsule networks}
To evaluate whether, as hypothesized, the discussed limitation affects the training of deep capsule negatively, we train the architecture from \cite{dynamic-routing} and vary the number of hidden capsule layers. We use $64$ primary capsules of dimension $8$, for each hidden layer $32$ capsules with dimension $12$, and for the output capsule layer a dimension of $16$. The mean and standard deviation of the test accuracy for $3$ runs with and without a bias term and for different depths are reported in tables \ref{tbl:deep_rba} and \ref{tbl:deep_em}.

First of all, it can be observed that \emph{a bias term enables the training of deeper networks} because the accuracy for RBA without a bias term drops already after $4$ layers and for EM without a bias term after $9$ layers. For EM routing we can additionally see that the accuracy drops for the more complex datasets if no bias term is used because the accuracy on smallNORB and SVHN drops already after $9$ layers. This supports our hypothesis from \Cref{sec:introduction} that the limitation we found will negatively affect the training of deep capsule networks. We can also see that in almost all cases \emph{the test accuracy is equal or higher if a bias term is used} for both, RBA and EM-routing. It is expected that the accuracy will not worsened when using bias terms since the optimizer can simply ignore the bias terms by setting all biases to zero.

To further evaluate the effect of a bias term, we compare training curves for EM-routing with and without a bias, trained on MNIST, in \cref{fig:em_curve}. It can be seen that the capsule network utilizing a bias immediately starts improving, independent of the depth of the network. The same network without a bias needs at least $1$k steps to start lowering its loss (\cref{fig:em_curve_7}). The improvement of the loss gets delayed further as the depth of the network increases (\cref{fig:em_curve_8}, \cref{fig:em_curve_9}). Therefore, it can be concluded that \emph{the training of capsule networks converges faster if a bias is used}. We also want to mention that the method we provide only needs a small number of additional parameters. For example, a bias term for a network with 5 layers increases the number of parameters by only $0.07 \permil$.

%%%%%%%%%%%%%%%%%%%%%%%%%%%%%%%%%%%%%%%%
\section{Conclusion}\label{sec:conclusion}
Although routing algorithms already proved successful in connecting capsules of different layers, as demonstrated by \citet{dynamic-routing} and \citet{em-routing}, we found a method to further improve both algorithms by increasing the expressivity of capsule networks. We theoretically proved the existence of a limitation and empirically showed that this negatively influences the training of deep capsule networks. Finally, we introduced a solution to this problem that (1) removes the aforementioned weakness, (2) enables the training of deep capsule networks, and (3) leads to a faster convergence of capsule networks while adding a negligible amount of parameters.

%%%%%%%%%%%%%%%%%%%%%%%%%%%%%%%%%%%%%%%%
\section*{Acknowledgments}
This research has received funding from the Horizon
2020 research and innovation programme under grant agreement no. 731761, IMAGINE.

%%%%%%%%%%%%%%%%%%%%%%%%%%%%%%%%%%%%%%%%
\bibliographystyle{model2-names}
\bibliography{refs}

\end{document}